\documentclass[twoside]{article}
\usepackage[round]{natbib}
\usepackage{jmlr2e}
\usepackage{caption}
 \usepackage{url}  
\usepackage{graphicx}  
\usepackage{amsfonts}
\usepackage{braket}
\usepackage{boxedminipage}
\usepackage{epsf}
\usepackage{bm}
\usepackage{amsmath}
\usepackage{xspace}
\usepackage{wrapfig}
\usepackage{algorithm,algpseudocode}
\algnewcommand{\Inputs}[1]{%
  \State \textbf{Inputs:}
  \Statex \hspace*{\algorithmicindent}\parbox[t]{.8\linewidth}{\raggedright #1}
}
\algnewcommand{\Initialize}[1]{%
  \State \textbf{Initialize:}
  \Statex \hspace*{\algorithmicindent}\parbox[t]{.8\linewidth}{\raggedright #1}
}
\def\indot<#1>{\langle #1 \rangle}

\begin{document}
%

\author{
 Daiki Suehiro\\
 \email{suehiro@ait.kyushu-u.ac.jp} \\
 \addr Faculty of Information Science \\and Electrical Engineering \\
 Kyushu University\\
and AIP, RIKEN\\
 Fukuoka, Japan, 8190395\\
\AND
 Kohei Hatano\\
 \email{hatano@inf.kyushu-u.ac.jp}\\
 \addr Faculty of Arts and Science\\
 Kyushu University\\
and AIP, RIKEN\\
 Fukuoka, Japan, 8190395\\
\AND
 Eiji Takimoto\\
 \email{eiji@inf.kyushu-u.ac.jp}\\
 \addr Department of Informatics\\
 Kyushu University\\
 Fukuoka, Japan, 8190395\\
\AND
 Shuji Yamamoto\\
 \email{yamashu@math.keio.ac.jp}\\
 \addr Department of Mathematics\\
 Keio University\\
and AIP, RIKEN\\
 Kanagawa, Japan, 2238522\\
\AND
 Kenichi Bannai\\
 \email{bannai@math.keio.ac.jp}\\
 \addr Department of Mathematics\\
 Keio University\\
and AIP, RIKEN\\
 Kanagawa, Japan, 2238522\\
\AND
 Akiko Takeda\\
 \email{takeda@mist.i.u-tokyo.ac.jp}\\
 \addr Department of Mathematical Informatics\\
 The University of Tokyo\\
and AIP, RIKEN\\
 Tokyo, Japan, 1138656\\
}

\editor{}

\title{Multiple-Instance Learning \\by Boosting Infinitely Many
  Shapelet-based Classifiers}
\maketitle

\newtheorem{defi}{Definition}
\newtheorem{theo}{Theorem}
\newtheorem{prop}[theo]{Proposition}
\newtheorem{coro}[theo]{Corollary}
\newtheorem{lemm}{Lemma}
\newtheorem{fact}{Fact}
\newtheorem{ex}{Example}
\def\remark{\par\noindent\hangindent0pt{\bf Remark.}~}


\def\OMIT#1{}
\def\newwd#1{{\em #1}}

\newcommand{\mnote}[1]{\marginpar{#1}}
\newcommand{\mynote}[1]{{\bf {#1}}}

%
%

\newcounter{nombre}
\renewcommand{\thenombre}{\arabic{nombre}}
\setcounter{nombre}{0}
\newenvironment{OP}[1][]{\refstepcounter{nombre}\par\bigskip \abovedisplayskip=0.5\abovedisplayskip \noindent{\sf OP \thenombre : #1}}{\par}

\global\long\def\T#1{#1^{\top}}

\newcommand{\shapelets}{shapelets}
\newcommand{\shapelet}{shapelet}
\newcommand{\LPS}[0]{LPM\xspace}
\newcommand{\targetH}{{local pattern matching hypothesis}}
\newcommand{\Nh}{N_{\mathrm{high}}}
\newcommand{\fsvm}{f_{\mathrm{SVM}}}
\newcommand{\frsvm}{f_{\mathrm{RSVM}}}
\newcommand{\OPTH}{\mathrm{OPT}_\textrm{hard}}
\newcommand{\OPTS}{\mathrm{OPT}_\textrm{soft}}
\newcommand{\OPTSNU}{\mathrm{OPT}_\mathrm{soft}(\nu^+)}
\newenvironment{claim}{\begin{trivlist}\item[]\textit{Claim}}{\end{trivlist}}
\newcommand{\AUC}{\mathrm{AUC}}
\newcommand{\bphi}{{\bf \phi}}
\newcommand{\bx}{{\bf x}}
\newcommand{\bsigma}{\boldsymbol{\sigma}}
\newcommand{\blambda}{\boldsymbol{\lambda}}
\newcommand{\kernel}{\boldsymbol{\mathrm{K}}}
\newcommand{\Vmat}{\boldsymbol{\mathrm{V}}}
\newcommand{\Xmat}{\boldsymbol{\mathrm{X}}}
\newcommand{\bw}{{\bf w}}
\newcommand{\bd}{{\bf d}}
\newcommand{\bk}{{\mathbf{k}}}
\newcommand{\bv}{{\bf v}}
\newcommand{\bu}{{\bf u}}
\newcommand{\bz}{{\bf z}}
\newcommand{\allins}{P_S}
\newcommand{\hatallins}{\hat{P}_S}
\newcommand{\multiallins}{P_S}
\newcommand{\op}{\textsf{OP}}
\newcommand{\bs}{{\bf s}}
\newcommand{\btau}{{\boldsymbol{\tau}}}
\newcommand{\bt}{{\bf t}}
\newcommand{\balpha}{\boldsymbol{\alpha}}
\newcommand{\bbeta}{\boldsymbol{\beta}}
\newcommand{\bgamma}{\boldsymbol{\gamma}}
\newcommand{\bxi}{\boldsymbol{\xi}}
\newcommand{\sbsq}{\mathrm{sub}}
\newcommand{\edge}{\mathrm{edge}}
\newcommand{\Ksub}{K_{\mathrm{sub}}}
\newcommand{\Ourmethod}[0]{our method\xspace}
\newcommand{\Ourshape}[0]{our shape.\xspace}
\newcommand{\hsigma}{\widehat\sigma}
\newcommand{\convhull}{\mathcal{H}}
\newcommand{\err}{\mathrm{err}}
\newcommand{\RW}{\mathrm{RW}}
\newcommand{\RCS}{\mathrm{RCS}}
\newcommand{\RV}{\mathrm{RV}}
\newcommand{\SRS}{\mathrm{SRS}}
\newcommand{\RSG}{\mathrm{RSG}}
\newcommand{\sign}{\mathrm{sign}}
\newcommand{\dom}{\mathcal{X}} 
\newcommand{\domp}{\mathcal{X}^{pos}} 
\newcommand{\domn}{\mathcal{X}^{neg}} %
\newcommand{\range}{\mathcal{Y}} 
\newcommand{\Natural}{\mathbb{N}} %
\newcommand{\Real}{\mathbb{R}} 
\newcommand{\Hilbert}{\mathbb{H}} 
\newcommand{\Prob}{\mathcal{P}}
\newcommand{\F}{\mathcal{F}}
\newcommand{\calS}{\mathcal{S}}
\newcommand{\calB}{\mathcal{B}}
\newcommand{\calF}{\mathcal{F}}
\newcommand{\calG}{\mathcal{G}}
\newcommand{\calT}{\mathcal{T}}
\newcommand{\calY}{\mathcal{Y}}
\newcommand{\Hyp}{\mathcal{H}}
\newcommand{\WH}{\mathcal{W}}
\newcommand{\EX}{\mathrm{EX}}
\newcommand{\filtEX}{\mathrm{FiltEX}}
\newcommand{\HSelect}{\mathrm{HSelect}}
\newcommand{\WL}{\mathrm{WL}}
\newcommand{\vecx}{\boldsymbol{x}}
\newcommand{\vecy}{\mbox{\boldmath $y$}}
\newcommand{\vecw}{\boldsymbol{w}}
\newcommand{\vecz}{\boldsymbol{z}}
\newcommand{\vecg}{\mbox{\boldmath $g$}}
\newcommand{\veca}{\mbox{\boldmath $a$}}
\newcommand{\vecd}{\boldsymbol{d}}
\newcommand{\vecell}{\mbox{\boldmath $\ell$}}
\newcommand{\vecsigma}{\boldsymbol{\sigma}}
\newcommand{\vecpi}{\boldsymbol{\pi}}
\newcommand{\vecxi}{\boldsymbol{\xi}}
\newcommand{\vece}{\mbox{\boldmath $e$}}
\newcommand{\vecB}{\mbox{\boldmath $B$}}
\newcommand{\vecD}{\mbox{\boldmath $D$}}
\newcommand{\vecI}{\mbox{\boldmath $I$}}
\newcommand{\tr}{\mathrm{tr}}
\newcommand{\vecG}{\mbox{\boldmath $G$}}
\newcommand{\vecF}{\mbox{\boldmath $F$}}
\newcommand{\tvecu}{\tilde{\mbox{\boldmath $u$}}}
\newcommand{\tvecw}{\tilde{\mbox{\boldmath $w$}}}
\newcommand{\tvecx}{\tilde{\mbox{\boldmath $x$}}}
\newcommand{\tw}{\tilde{w}}
\newcommand{\tx}{\tilde{x}}
\newcommand{\haty}{\hat{y}}
\newcommand{\vecf}{\mbox{\boldmath $f$}}
\newcommand{\vectheta}{\mbox{\boldmath $\theta$}}
\newcommand{\vecalpha}{\boldsymbol{\alpha}}
\newcommand{\vecbeta}{\mbox{\boldmath $\beta$}}
\newcommand{\vectildealpha}{\widetilde{\vecalpha}}
\newcommand{\vectildebeta}{\widetilde{\vecbeta}}
\newcommand{\tildealpha}{\widetilde{\alpha}}
\newcommand{\tildebeta}{\widetilde{\beta}}
\newcommand{\vechatalpha}{\widehat{\vecalpha}}
\newcommand{\vechatbeta}{\widehat{\vecbeta}}
\newcommand{\hatalpha}{\widehat{\alpha}}
\newcommand{\hatbeta}{\widehat{\beta}}
\newcommand{\vectau}{\mbox{\boldmath $\tau$}}
\newcommand{\veclambda}{\bm{\lambda}}
\newcommand{\vecu}{\mbox{\boldmath $u$}}
\newcommand{\vecv}{\mbox{\boldmath $v$}}
\newcommand{\vecp}{\boldsymbol{p}}
\newcommand{\vecq}{\mbox{\boldmath $q$}}
\newcommand{\vecr}{\boldsymbol{r}}
\newcommand{\vecc}{\boldsymbol{c}}
\newcommand{\fp}{\mathrm{fp}}
\newcommand{\fn}{\mathrm{fn}}
\newcommand{\ouralg}{{Our algorithm}~}
\newcommand{\Ouralg}{PUMMA~}

\newcommand{\bn}{\Delta_2} 
\newcommand{\psimp}{\mathcal{P}} %
\newcommand{\hatgamma}{\hat{\gamma}}
\newcommand{\indctr}[1]{I(#1)}
\newcommand{\CLASS}{\mathcal{C}}
\newcommand{\VC}{\mathrm{VC}}

\newcommand{\reg}{\mathcal{R}}
\newcommand{\breg}{D}

\newcommand{\filtex}{\mathrm{GenD_t}}
\newcommand{\gensamp}{\mathrm{GenSample}}

\newcommand{\argmax}{\mathop{\rm arg~max}\limits}
\newcommand{\argmin}{\mathop{\rm arg~min}\limits}
\newcommand{\Expo}{\mathop{\rm  E}\limits}

\newcommand{\half}{\frac{1}{2}}
\newcommand{\eps}{\varepsilon}

\newcommand{\hp}{\hat{p}}
\newcommand{\hmup}{\hat{\mu}[+]}
\newcommand{\hmun}{\hat{\mu}[-]}
\newcommand{\hgp}{\hat{\gamma}[+]}
\newcommand{\hgn}{\hat{\gamma}[-]}
\newcommand{\gain}{\Delta}
\newcommand{\hgain}{\hat{\Delta}}
\newcommand{\vecdelta}{\boldsymbol{\delta}}

\newcommand{\tildeO}{\Tilde{O}}
\newcommand{\permset}{S}
\newcommand{\base}{\boldsymbol{B}}
\newcommand{\calC}{\mathcal{C}}
\newcommand{\calP}{\mathcal{P}}
\newcommand{\calX}{\mathcal{X}}
\newcommand{\calE}{\mathcal{E}}
\newcommand{\Rdm}{\mathfrak{R}}
\newcommand{\GC}{\mathfrak{G}}
\newcommand{\hullC}{\mathrm{conv}(\calC)}
\newcommand{\hullH}{\mathrm{conv}(H)}
\newcommand{\conv}{\mathrm{conv}}

\newcommand{\Rmin}{R_{\mathrm{min}}}
\newcommand{\Rmax}{R_{\mathrm{max}}}

\def\ceil#1{%
\left\lceil #1 \right\rceil}

\def\defeq{%
\stackrel{\mathrm{def}}{=}}

\def\floor#1{%
\lfloor #1 \rfloor}

\def\myhang{%
    \par\noindent\hangindent20pt\hskip20pt}
\def\nitem#1{%
    \par\noindent\hangindent40pt
    \hskip40pt\llap{#1~}}

\newcommand{\E}{\boldsymbol{E}}

\newcommand{\pdiff}{\Phi_{\mathrm{diff}}(\multiallins)}
\newcommand{\calI}{\mathcal{I}}

 \begin{abstract}
We propose a new formulation of  Multiple-Instance Learning (MIL).
In typical MIL settings,  a unit of data is given as a set of instances
  called a bag and the goal is to find a good classifier of bags based on
  similarity from a single or finitely many ``shapelets'' (or patterns), where the similarity
  of the bag from a shapelet is the maximum similarity of
  instances in the bag.
  Classifiers based on a single shapelet are not sufficiently strong for
  certain applications.   
  Additionally, previous work with multiple shapelets has
  heuristically chosen some of the instances as
  shapelets with no theoretical guarantee of its generalization ability.
Our formulation provides a richer class of the final classifiers based on infinitely many shapelets.
We provide an efficient algorithm for the new formulation, in addition to
  generalization bound. 
Our empirical study demonstrates 
that our approach is effective not only for MIL tasks
  but also for Shapelet Learning for time-series classification\footnote{The preliminary version of this
    paper is~\cite{Suehiro2017BoostingTK}, which only focuses on
    shapelet-based time-series classification but not
    Muptiple-Instance Learning. Note that the preliminary version has not been
    published.}.

 \iffalse
 We investigate the connection of two learning problems,
Multiple-Instance Learning (MIL) and Shapelet-Learning (SL).
MIL is a popular learning framework in machine learning 
community, in which data is given as the set of instances called bag.
SL is recently proposed in time-series analysis community
that learn local features defined by some ``short'' series called shapelets.
We show that SL problem can be reduced to our proposed MIL
formulation, and design theoretically guaranteed learning algorithms
based on boosting with kernelized base classifier. 
Our main contribution is the problem formulation and algorithms that
allow us to learn shapelet-like hypothesis in MIL. 
Moreover, we give theoretical justification about 
the generalization performance of existing 
shapelet-learning algorithms that are only been shown
empirical effective.
Our algorithms perform favorably with baselines in our empirical study.

 \iflase
 We propose a new formulation of  Multiple-Instance Learning (MIL),
  which is motivated by Shaplet Learning (SL) for time-series classification.
 In typical MIL settings,  a unit of data is given as a set of instances
  called a bag and the goal is to find a good classifier of bags based on
  similarity from some ``pattern'', where the similarity
  of the bag from the pattern is the maximum of similarity of
  instances in the bag.
  On the other hand, in typical SL settings, the goal is to find several
  patterns, from which a good ensemble of similarity-based classifiers
  can be obtained.
  By viewing SL as an MIL, the resulting formulation provides a richer
  class of the final classifier based on infinitely many patterns.
We give an efficient algorithm for the new formulation, as well as
  generalization bound. 
Our empirical study shows that our approach is effective not only for SL tasks
  but also other MIL tasks.
 \fi

 \end{abstract}
\section{Introduction}
Multiple-Instance Learning (MIL) is a fundamental framework of
supervised learning with a wide range of applications such as
prediction of molecule activity, image classification, and so on.
Since the notion of MIL was first proposed by~\citet{Dietterich:1997},
MIL has been extensively studied both in theoretical and practical aspects
\citep{Gartner02multi-instancekernels,NIPS2002misvm,Sabato:2012:MLA,pmlr-v28-zhang13a,Doran:2014,CARBONNEAU2018329}.

A standard MIL setting is described as follows:
A learner receives sets $B_1, B_2, \ldots, B_m$
called bags, each of which contains multiple instances. 
In the training phase, each bag is labeled
but instances are not labeled individually.
The goal of the learner is to obtain a hypothesis 
that predicts the labels of 
unseen bags correctly\footnote{Although there are settings where
instance label prediction is also considered, 
we focus only on bag-label prediction in this paper.}.
One of the most common hypotheses used in practice has the following form:
\begin{align}
\label{align:single-shape}
h_\bu(B) = \max_{x \in B} \left\langle \bu, \Phi(x)\right\rangle,
\end{align}
where $\Phi$ is a feature map and 
$\bu$ is a feature vector which we call a \emph{shapelet}.
In many applications, $\bu$ is
interpreted as a particular ``pattern'' in the feature space
and the inner product
as the similarity of $\Phi(x)$ from $\bu$.
Note that we use the term ``shapelets'' by following the terminology
of Shapelet Learning, which is a framework for time-series classification, 
although it is often called ``concepts'' in the literature of MIL.
Intuitively, this hypothesis evaluates a given bag by the
maximum similarity of the instances in the bag from the shapelet $\bu$.
Multiple-Instance Support Vector Machine (MI-SVM) proposed by~\citet{NIPS2002misvm} 
is a widely used algorithm that
uses this hypothesis class and 
learns $\bu$.
It is well-known that MIL algorithms using this hypothesis class
practically perform well for various multiple-instance datasets.
Moreover, a generalization error bound of the hypothesis class 
is given by~\citet{Sabato:2012:MLA}.

However, in some domains such as 
image recognition and document classification,
it is said that the hypothesis class of (\ref{align:single-shape}) is
not effective enough~\citep[see, e.g.,][]{MI1normSVM}.
To employ MIL on such domains more effectively,
\citet{MI1normSVM} proposes to use
a convex combination of various shapelets $\bu$ in a finite set
$U=\{\Phi(z) \mid z \in \bigcup_{i=1}^n B_i\}$, which is defined based 
on all instances that appear in the training sample,
\begin{align}
\label{align:our-hypo}
 	 g(B) = \sum_{\bu \in U} w_\bu \max_{x \in B} \left\langle \bu,
  \Phi(x)\right\rangle, 
\end{align}
where $\bw$ is a probability vector over $U$.   
They demonstrate that
this hypothesis with the Gaussian kernel performs well
in image recognition.
However, no theoretical justification is known for the hypothesis class of type
(\ref{align:our-hypo}) with the finite set $U$ made from the empirical bags.
%
By contrast,
for sets $U$ of infinitely many shapelets $\bu$ with bounded norm,
generalization bounds of \citet{Sabato:2012:MLA} are
applicable as well to the hypothesis class (\ref{align:our-hypo}),
but the result of \citet{Sabato:2012:MLA} does not provide a practical formulation such as MI-SVM. 

\subsection{Our Contributions}
In this paper, 
we propose an MIL formulation with the hypothesis class
(\ref{align:our-hypo}) for sets $U$ of infinitely many shapelets.
More precisely,
we formulate a $1$-norm regularized 
soft margin maximization problem 
to obtain linear combinations of shapelet-based hypotheses.

Then, we design an algorithm 
based on Linear Programming Boosting
\citep[LPBoost, ][]{demiriz-etal:ml02} that solves the soft margin optimization problem
via a column generation approach.
Although
the sub-problems (weak learning problem) become optimization
problems over an infinite-dimensional space,
we can show that an analogue of the representer theorem holds on it 
and allows us to reduce it to a non-convex optimization problem
(difference of convex program, DC-program for short)
over a finite-dimensional space.
While it is difficult to solve the sub-problems exactly due to
non-convexity,
various techniques \citep[e.g.,][]{Tao1988,Yu:2009:LSS} are investigated for DC programs and we can find 
good approximate solutions efficiently for many cases in practice.

Furthermore, we prove a generalization error bound of hypothesis class
(\ref{align:our-hypo}) with infinitely large sets $U$.
In general, our bound is incomparable with those of
\citet{Sabato:2012:MLA}, but ours has better rate in terms of the sample size
$m$.

We introduce an important application of our result,
shapelet learning for time-series classification (we show details later).
In fact, in time-series domain, 
most shapelet learning algorithms 
have been designed heuristically. 
As a result, our proposed algorithm
becomes the first algorithm for shapelet learning in
time-series classification that guarantees 
the theoretical generalization performance.

Finally, the experimental results show that our approach
performs favorably with a baseline for 
shapelet-based time-series classification tasks
and outperforms baselines for several MIL tasks.

\subsection{Comparison to Related Work}
There are many MIL algorithms 
with hypothesis classes which are different from
(\ref{align:single-shape}) or (\ref{align:our-hypo}). 
\citep[e.g., ][]{Gartner02multi-instancekernels,NIPS2005_2926,MI1normSVM}.
Many of them adopt different approaches for the bag-labeling hypothesis
from shapelet-based classifiers 
(e.g., \citet{NIPS2005_2926} used a Noisy-OR based hypothesis and 
\citet{Gartner02multi-instancekernels} proposed a new kernel
called a set kernel).

\citet{Sabato:2012:MLA} proved generalization bounds of hypotheses
classes for MIL including those of ~(\ref{align:single-shape}) and (\ref{align:our-hypo}) with infinitely large
sets $U$.
They also proved the PAC-learnability of the class
(\ref{align:single-shape}) using the boosting approach under some
technical assumptions.
Their boosting approach is different from our work in that they assume
that labels are consistent with some hypothesis of the
form~(\ref{align:single-shape}), while we consider arbitrary distributions
over bags and labels.

It is known that other boosting-based methods achieve successful results
for several MIL tasks~\citep{mi_adaboost,NIPS2003DPBOOST,NIPS2005_2926}.
They use different hypothesis classes than ours under different assumptions. 

  \subsection{Connection between MIL and Shapelet Learning for Time Series
  Classification}

Here we briefly mention that MIL with
type (\ref{align:our-hypo}) hypotheses
is closely related to Shapelet Learning (SL), which is a framework
for time-series classification and has been extensively
studied by~\citep{Ye:2009:TSS:1557019.1557122,KeoghR13,Hills:2014:CTS:2597434.2597448,Grabocka:2014:LTS:2623330.2623613}
in parallel to MIL.
SL is a notion of learning with a particular method of feature extraction,
which is defined by
a finite set $M \subseteq \Real^\ell$ of
real-valued ``short'' sequences called shapelets and
a similarity measure (not necessarily a Mercer kernel)
$K: \Real^\ell \times \Real^\ell \to \Real$
in the following way.
A time series $\btau = (\tau[1], \dots, \tau[L]) \in \Real^L$
can be identified with a bag
$B_\btau = \{(\tau[j], \ldots, \tau[j+\ell-1])
\mid 1 \leq j \leq L-\ell+1\}$ consisting of all subsequences
of $\btau$ of length $\ell$.
Then, the feature of $\btau$ is a vector
	$\left(\max_{\bx \in B_\btau} K(\bz, \bx) \right)_{\bz \in M}$
of a fixed dimension $|M|$ regardless of the length
$L$ of the time series $\btau$.
When we employ a linear classifier on top of the features, we obtain
a hypothesis of the form
\begin{align}
\label{align:multi-shape-ts}
	g(\btau) = \sum_{\bz \in M} w_\bz \max_{\bx \in B_\tau} K(\bz,\bx),
\end{align}
which is essentially the same form as~(\ref{align:our-hypo}),
except that
finding good shapelets $M$ is a part of the learning task,
as well as to finding good weight vector $\bw$.
This is one of the most successful approach of SL
\citep{Hills:2014:CTS:2597434.2597448,Grabocka:2014:LTS:2623330.2623613,GrabockaWS15,renard:hal-01217435,HouKZ16}, where
a typical choice of $K$ is $K(\bz,\bx) = -\|\bz - \bx\|_2$.
However, almost all existing methods heuristically choose
shapelets $M$ and have no theoretical guarantee on how good
the choice of $M$ is.

Note also that in the SL framework, each $\bz \in M$ is called a
shapelet, while in this paper, we assume that $K$ is a kernel
$K(z,x) = \langle \Phi(z), \Phi(x) \rangle$ and
any $\bu$ (not necessarily $\Phi(z)$ for some $z$)
in the Hilbert space is called a shapelet.


Curiously, despite MIL and SL share similar motivations and
hypotheses, the relationship between MIL and SL has not yet been 
pointed out.
From the shapelet-perspective in MIL, the hypothesis (\ref{align:single-shape})
is regarded as a ``single shapelet''-based hypothesis, and
the hypothesis (\ref{align:our-hypo}) is regarded as 
``multiple shapelet''-based hypothesis.
We refer to a linear combination of maximum similarities based on
shapelets such as (\ref{align:our-hypo}) and
(\ref{align:multi-shape-ts})
as {\it shapelet-based classifiers}.


\section{Preliminaries}
\label{sec:prelim}
Let $\calX$ be an instance space.
A bag $B$ is a finite set of instances chosen from $\calX$. 
The learner receives a sequence of labeled bags
$S = ((B_1, y_1), \ldots, (B_m, y_m)) \in
(2^{\calX} \times \{-1, 1\})^m$ called a sample, where 
each labeled bag is independently drawn according to some unknown
distribution $D$ over $2^{\calX} \times \{-1, 1\}$.
Let $\allins$ denote the
set of all instances that appear in the sample $S$.
That is, $\allins = \bigcup_{i=1}^m B_i$.
Let $K$ be a kernel over $\calX$, which is used to measure the
similarity between instances, and let
$\Phi: \calX \to \Hilbert$ denote a feature map associated with
the kernel $K$ for a Hilbert space $\Hilbert$,
that is,
$K(z, z')=\langle \Phi(z), \Phi(z')\rangle$
for instances $z, z' \in \calX$, where
$\langle \cdot,\cdot \rangle$ denotes the inner product over $\Hilbert$.
The norm induced by the inner product is denoted by
$\|\cdot\|_\Hilbert$ defined as $\|\bu\|_{\Hilbert} =
\sqrt{\langle \bu, \bu \rangle}$ for $\bu \in \Hilbert$. 

For each $\bu \in \Hilbert$ which we call a shapelet,
we define 
a {\it shapelet-based classifier}
denoted by $h_{\bu}$, 
as the function that maps a given bag
$B$ to the maximum of
the similarity scores between shapelet $\bu$ and $\Phi(x)$
over all instances $x$ in $B$.
More specifically,
\[
	h_{\bu}(B) = \max_{x \in B}
		\left\langle\bu, \Phi(x) \right\rangle.
\]
For a set $U \subseteq \Hilbert$, we define the class of shapelet-based classifiers
as 
\[
	H_U = \left\{ h_{\bu} \mid \bu \in U \right\}
\]
and let $\conv(H_U)$ denote the set of convex combinations of
shapelet-based classifiers in $H_U$. More precisely,
\begin{align*}
	 \conv(H_U) = 
       &\left\{
		\int_{\bu \in U} w_\bu h_\bu d\bu \mid
		\text{$w_\bu$ is a density over $U$}
          \right\} \\
       =&\left\{
		\sum_{\bu \in U'} w_\bu h_\bu \mid
		\forall \bu \in U', w_\bu \geq 0, \right. \\
		&\left. \sum_{\bu \in U'} w_\bu = 1,
		\text{$U' \subseteq U$ is a finite support} 
 	\right\}.
\end{align*}
The goal of the learner is to find a hypothesis
$g \in \conv(H_U)$, so that its generalization error
$\calE_D(g) = \Pr_{(B,y) \sim D}[\sign(g(B)) \neq y]$
is small.
Note that since the final hypothesis $\sign \circ g$
is invariant to any scaling of $g$, we assume
without loss of generality that
\[
	U = \{\bu \in \Hilbert \mid \|\bu\|_\Hilbert \leq 1\}.
\]
Let
$\calE_{\rho}(g)$ denote the \textit{empirical margin loss} of $g$
over $S$,
that is, $\calE_{\rho}(g) = |\{i \mid y_i g(B_i) < \rho\}|/m$.



\section{Optimization Problem Formulation}
In this paper we formulate the problem as
soft margin maximization with $1$-norm regularization,
which ensures a generalization bound
for the final hypothesis~\citep[see, e.g.,][]{demiriz-etal:ml02}.
Specifically, the problem is formulated as a
linear programming problem (over infinitely many variables) as
follows:
\begin{align}\label{align:LPBoostPrimal}
\max\limits_{\rho,w,\bxi} \; &
 \rho  - \frac{1}{\nu m}\sum_{i=1}^m \xi_{i}
\\ \nonumber
 \text{sub.to} \; 
& \int_{\bu \in U}y_i w_{\bu}h_\bu(B_i) 
d\bu 
                       \geq \rho -\xi_{i} \wedge \xi_{i} \geq0, ~ i \in [m],
\\ \nonumber
&\int_{\bu \in U} w_{\bu}d\bu  = 1,
 w_\bu  \geq 0,~ \rho \in \Real,
\end{align}
where $\nu \in [0,1]$ is a parameter.
To avoid the integral over the Hilbert space,
it is convenient to consider the dual form:
\begin{align}\label{align:LPBoostDual}
\min\limits_{\gamma,\bd} \; & \gamma
\\ \nonumber
\text{sub.to} \; 
& \sum_{i=1}^my_id_i h_\bu(B_i) \leq \gamma,
\; \bu \in U,
\\ \nonumber
& 0 \leq d_i \leq 1/(\nu m),\; i \in [m],
\\ \nonumber
& \sum_{i=1}^m d_{i}  = 1, ~ \gamma \in \Real.
\end{align}
The dual problem is categorized as a semi-infinite program (SIP) because
it contains infinitely many constraints.
Note that the duality gap is zero because the problem
(\ref{align:LPBoostDual})
is linear and the optimum is finite \citep[Theorem 2.2 of][]{Shapiro09}.
We employ column generation to solve the dual problem:
solve (\ref{align:LPBoostDual}) for a finite subset $U' \subseteq U$,
find $\bu$ to which the corresponding constraint is maximally violated
by the current solution (\emph{column generation part}),
and repeat the procedure with $U' = U' \cup \{\bu\}$ 
until a certain stopping criterion is met.
In particular, we use
LPBoost~\citep{demiriz-etal:ml02}, a well-known and practically
fast algorithm of column generation.
Since the solution $\bw$ is expected to be sparse
due to the 1-norm regularization,
the number of iterations is expected to be small.

Following the terminology of boosting, we refer to the
column generation part as weak learning.
In our case, weak learning is formulated as the following
optimization problem:
\begin{align}\label{align:WeakLearn_u}
\max_{\bu \in \Hilbert} \;
\sum_{i=1}^my_id_i
\max_{x \in B_i}  \left\langle \bu,  \Phi\left(x \right)
                       \right\rangle \; \text{sub.to} \; \|\bu\|_\Hilbert^2 \leq 1.
\end{align}
Thus, we need to design a weak learner for solving
(\ref{align:WeakLearn_u}) for a given sample weighted by $\bd$.
It seems to be impossible to solve it directly because
we only have access to $U$ through the associated kernel.
Fortunately, we prove a version of representer theorem given below,
which makes (\ref{align:WeakLearn_u}) tractable.
\begin{theo}[Representer Theorem]\label{theo:represent}
The solution $\bu^*$ of (\ref{align:WeakLearn_u}) can be
written as
$\bu^* = \sum_{z \in \allins} \alpha_z \Phi(z)$
for some real numbers $\alpha_z$.
\end{theo}
Our theorem can be derived from an application 
of the standard representer theorem \citep[see, e.g.,][]{Mohri.et.al_FML}.
Intuitively, we prove the theorem by decomposing the optimization problem
(\ref{align:WeakLearn_u})
into a number of sub-problems, so that the standard representer theorem
can be applied to each of the sub-problems.
The detail of the proof is given in the supplementary materials.
Note that Theorem~\ref{theo:represent} gives justification 
to the simple heuristics in the literature: choosing the shapelets
extracted from $P_S$.

Theorem~\ref{theo:represent} says that the weak learning problem
can be rewritten as the following tractable form:
\begin{OP}\label{align:WeakLearn} {\bf Weak Learning Problem}
\begin{align*}
\min_{\balpha} \;& 
- \sum_{i=1}^m{d}_i y_i
\max_{x \in B_i} \sum_{z \in \allins} \alpha_{z} K\left(
                        z, x \right)
\\ \nonumber
\text{sub.to} \; & 
\sum_{z \in \allins}\sum_{v \in \allins}\alpha_{z}\alpha_{v}K \left(z, v \right)  \leq 1.
\end{align*}
\end{OP}
Unlike the primal solution $\bw$, the dual solution
$\balpha$ is not expected to be sparse.
In order to obtain a more interpretable hypothesis,
we propose another formulation of weak learning where
1-norm regularization is imposed on $\balpha$, so that
a sparse solution of $\balpha$ will be obtained.
In other words, instead of $U$, we consider the feasible set
$\hat{U} = \left\{\sum_{z \in \allins}\alpha_{z} \Phi(z):
\|\balpha\|_1 \leq 1\right\}$,
where $\|\balpha\|_1$ is the 1-norm of $\balpha$.
\begin{OP}\label{align:WeakLearnSP} {\bf Sparse Weak Learning Problem}
\begin{align*}
\min_{\balpha} \; & 
- \sum_{i=1}^m{d}_i y_i
\max_{x \in B_i} \sum_{z \in \allins} \alpha_{z} K\left(
                        z, x\right)
\\
\text{sub.to} \;\; & \|\balpha\|_1  \leq 1
\end{align*}
\end{OP}
Note that when running LPBoost with a weak learner for
\op~\ref{align:WeakLearnSP}, we obtain a final hypothesis
that has the same form of generalization
bound as the one stated in Theorem~\ref{theo:main}, which
is of a final hypothesis obtained when used with a weak learner
for \op~\ref{align:WeakLearn}.
To see this, consider a feasible space
$\hat U_\Lambda = \left\{\sum_{z \in \allins}\alpha_{z} \Phi(z):
\|\balpha\|_1 \leq \Lambda \right\}$
for a sufficiently small $\Lambda > 0$, so that
$\hat U_\Lambda \subseteq U$. Then since
$H_{\hat U_\Lambda} \subseteq H_U$, a generalization bound
for $H_U$ also applies to $H_{\hat U_\Lambda}$.
On the other hand, since the final hypothesis $\sign \circ g$
for $g \in \conv(H_{\hat U_\Lambda})$ is invariant to the
scaling factor $\Lambda$, the generalization ability is independent
of $\Lambda$.


%


\section{Algorithms}

For completeness, we present the pseudo code of LPBoost
in Algorithm~\ref{alg:LPBoost}.

For the rest of this section,
we describe our algorithms for the weak learners.
For simplicity, we denote by $\bk_x \in \Real^{P_S}$
a vector given by $k_{x,z} = K(z,x)$ for every $z \in P_S$.
Then, the objective function of \op~\ref{align:WeakLearn}
(and \op~\ref{align:WeakLearnSP})
can be rewritten as
\[
	\sum_{i:y_i = -1} d_i \max_{x \in B_i} \bk_x^T \balpha
	- \sum_{i:y_i = 1} d_i \max_{x \in B_i} \bk_x^T \balpha,
\]
which can be seen as a difference $F - G$ of two convex functions 
$F$ and $G$ of $\balpha$.
Therefore, the weak learning problems are DC programs and thus
we can use DC algorithm~\citep{Tao1988,Yu:2009:LSS}
to find an $\epsilon$-approximation of a local optimum.
We employ a standard DC algorithm. That is, for each
iteration $t$, we linearize the concave term $G$
with $\nabla_{\balpha} G(\balpha_t)^T \balpha$ at
the current solution $\balpha_t$, which is 
$\sum_{i:y_i=1} d_i \bk_{x_i}^T \balpha$ with
$x_i = \arg\max_{x \in B_i} \bk_x^T \balpha$ in our case,
and then update the solution to $\balpha_{t+1}$ by solving the resultant
convex optimization problem $\op'_t$.

In addition, the problems $\op'_t$
for \op~\ref{align:WeakLearn} and \op~\ref{align:WeakLearnSP}
are reformulated as a second-order cone programming (SOCP) problem
and an LP problem, respectively, and thus both
problems can be efficiently solved.
To this end, we introduce new variables
$\lambda_i$ for all negative bags $B_i$ with $y_i = -1$
which represent the factors
$\max_{x \in B_i} \bk_x^T \balpha$.
Then we obtain the equivalent problem to $\op'_t$
for \op~\ref{align:WeakLearn}
as follows:
\begin{align}\label{align:WeakLearnSub}
\min_{\balpha, \blambda} \; & 
\sum_{i:y_i= -1}d_i \lambda_i
- \sum_{i:y_i = 1} d_i \max_{x \in B_i}\bk_{x_i}^T \balpha 
\\ \nonumber
\text{sub.to} \; &  \bk_x^T \balpha
                      \leq \lambda_i 
                      ~(\forall i:y_i=-1, \forall x \in B_i), \\ \nonumber
&\sum_{z \in \allins}\sum_{v \in \allins}\alpha_{z}\alpha_{v}K \left(z, v \right)  \leq 1.
\end{align}
It is well known that this is an SOCP problem.
Moreover, it is clear that $\op'_t$ for
\op~\ref{align:WeakLearnSP} can be formulated as an LP problem.
We describe the algorithm for \op~\ref{align:WeakLearn}
in Algorithm~\ref{alg:WeakLearn}.

\begin{algorithm}
\caption{LPBoost using WeakLearn}
\label{alg:LPBoost}
\begin{algorithmic}[0]
\Inputs{$S$, kernel $K$, $\nu \in (0, 1]$, $\epsilon > 0$}
\Initialize{$\bd_0 \leftarrow (\frac{1}{m}, \ldots, \frac{1}{m}), \gamma=0$}
\For{ $t=1,\dots $}
 \State $h_t\leftarrow$  Run {\bf WeakLearn}($S$, $K$, $\bd_{t-1}$,
   $\epsilon$)
 \If {$\sum_{i=1}^my_id_i h_t(B_i) \leq \gamma$}
 \State $t = t-1$, break
 \EndIf
\begin{align}
\nonumber
& (\gamma, \bd_t) \leftarrow \arg\min\limits_{\gamma,\bd} \;  \gamma 
\\ \nonumber
&\text{sub.to}~~ 
 \sum_{i=1}^my_id_i h_j(B_i) \leq \gamma
~~(j = 1, \ldots, t), ~~
\\ \nonumber
&0 \leq d_i \leq 1/\nu m ~~(i \in [m]),  \sum_{i=1}^m d_i  = 1, ~ \gamma \in \Real.
\end{align}
\EndFor
\State $\bw \leftarrow$ Lagrangian multipliers of the last solution 
\State $g \leftarrow \sum_{j=1}^t w_j h_j$ \\
\Return {$\sign(g)$}
\end{algorithmic}
\end{algorithm}

\begin{algorithm}[ht]
\caption{WeakLearn using the DC Algorithm}
\label{alg:WeakLearn}
\begin{algorithmic}[0]
\Inputs{$S$, $K$, 
    $\bd$, $\epsilon$ (convergence parameter)}
 \Initialize{$\balpha_0 \in \Real^{|\allins|}$, $f_0 \leftarrow \infty$}
 \For{$t=1,\dots $}
\For{$\forall k:y_k=+1$}
\State 
$\displaystyle x^*_k \leftarrow \arg\max_{x \in B_k} \sum_{z \in \allins}
   d_k\alpha_{z} K\left(z, x\right)$ 
\EndFor 
\State
\begin{minipage}{6cm}
\begin{align}\label{align:subprob}
\nonumber
 f \leftarrow \min_{\balpha, \boldsymbol{\lambda}} 
\;& 
- \sum_{k:y_k=+1} {d}_k
\sum_{z \in \allins} \alpha_{z} K\left(z, x_k^*\right)
\\ 
&+ \sum_{r:y_r=-1}{d}_r \lambda_r 
\\ \nonumber
\text{sub.to} \;&  \sum_{z \in \allins} \alpha_{z} K\left(z, x  \right) 
                      \leq \lambda_r 
  \\ \nonumber
                      &
~(\forall r:y_r=-1, \forall x \in B_r),\\ \nonumber
&\sum_{z \in \allins}\sum_{v \in \allins} \alpha_{z}\alpha_{v}K
  \left(z, v \right)  \leq 1.
\end{align}
\end{minipage}
\State $\balpha_t \leftarrow \balpha$, $f_t \leftarrow f$
\If{$f_{t-1} - f_{t} \leq \epsilon$} 
\State break
\EndIf 
\EndFor \\
\Return{$h(B) =  \max_{x \in
    B}\sum_{z \in \allins}\alpha_{z}K(z, x)$}
\end{algorithmic}
\end{algorithm}

\section{Generalization Bound of the Hypothesis Class}
\label{sec:theorem1}
In this section, we provide a generalization bound of
hypothesis classes $\conv(H_U)$ for various $U$ and $K$.

Let $\Phi(\multiallins)=\{\Phi(z) \mid z \in \multiallins\}$.
Let $\pdiff=\{\Phi(z) - \Phi(z') \mid
z,z'\in \multiallins, z \neq z'\}$.
By viewing each instance $\bv \in \pdiff$ as a hyperplane
$\{\bu \mid \langle \bv, \bu \rangle = 0\}$,
we can naturally define a partition
of the Hilbert space $\Hilbert$
by the set of all hyperplanes $\bv \in \pdiff$.
Let $\calI$ be the set of all cells of the partition, i.e.,
$\calI=\{ I \mid
I=\cap_{\bv \in V}
\{\bu \mid  \langle \bv, \bu \rangle \geq 0\}, I \neq \emptyset, \text{$\forall \bv
\in \pdiff$, $V$
contains either $\bv\in \pdiff$ or}$ $\text{$-\bv \in \pdiff$ }
\}$.
Each cell $I \in \calI$ is a polyhedron which is defined by
a minimal set $V_I \subseteq \pdiff$ that
satisfies $I = \bigcap_{\bv \in V_I}
\{\bu \mid \langle \bu, \bv \rangle \geq 0\}$.
Let
\[
	\mu^* = \min_{I \in \calI}\max_{\bu\in I \cap U}\min_{\bv \in V_I}
 |\langle \bu, \bv \rangle|. 
\]
Let $d^*_{\Phi,S}$ be the VC dimension of the set of linear classifiers
over the finite set $\pdiff$, given by
$F_U=\{ f: \bv \mapsto \sign(\langle \bu, \bv\rangle) \mid \bu \in U\}$.

Then we have the following generalization bound on the hypothesis class
of (\ref{align:our-hypo}).

\begin{theo}
\label{theo:main}
Let $\Phi: \calX \rightarrow \Hilbert$.
Suppose that for any $z \in \calX$, $\|\Phi(z)\|_\Hilbert \leq R$.
Then, for any $\rho>0$, 
with high probability 
the following holds
for any $g \in \conv(H_U)$ with
$U \subseteq \{\bu \in \Hilbert \mid \|\bu\|_\Hilbert \leq 1\}$:
\begin{align}
 \calE_D(g) \leq& \calE_{\rho}(g)
 +O\left(
 \frac{R \sqrt{d_{\Phi, S}^* \log |\allins|}}{\rho\sqrt{m}}  
\right),
\end{align}
where 
 (i) for any $\Phi$,
 $d^*_{\Phi,S}=O((R/\mu^*)^2)$,
(ii) if $\calX \subseteq \Real^\ell$ 
and $\Phi$ is the identity mapping (i.e., the associated kernel
 is the linear kernel), or
(iii) if $\calX \subseteq \Real^\ell$ 
and $\Phi$ satisfies the condition that
$\left\langle\Phi(z), \Phi(x) \right\rangle$ is monotone
decreasing with respect to $\|z-x\|_2$  (e.g.,
the mapping defined by the Gaussian kernel) and
$U=\{\Phi(z) \mid z \in \Real^\ell, \|\Phi(z)\|_\Hilbert \leq 1\}$, 
 then $d^*_{\Phi, S}=O(\min((R/\mu^*)^2, \ell))$.
\end{theo}
For space constraints, we omit the proof and it is shown in the supplementary materials.

\paragraph{Comparison with the existing bounds}
A similar generalization bound 
can be derived from 
a known bound of the Rademacher complexity of $H_U$~\citep[Theorem 20 of][]{Sabato:2012:MLA}
and a generalization bound of $\hullH$ for any hypothesis class $H$
~\citep[see Corollary 6.1 of][]{Mohri.et.al_FML}:
\[
\calE_D(g) \leq \calE_{\rho}(g) + O\left(\frac{{\log
      \left(\sum_{i=1}^m|B_i| \right)\log(m)}}{\rho\sqrt{m}} \right).
\]
Note that \citet{Sabato:2012:MLA} fixed $R=1$. 
Here, for simplicity, we omit some constants of~\citep[Theorem 20 of][]{Sabato:2012:MLA}. 
Note that $|\allins| \leq \sum_{i=1}^m|B_i|$ by definition.
The bound above is incomparable to
Theorem~\ref{theo:main} in general, as ours uses the parameter $d^*_{\Phi,S}$ and
the other has the extra $\sqrt{\log\left(\sum_{i=1}^m|B_i|
  \right)}\log(m)$ term.
However, our bound is better in terms of the sample size $m$ by the
factor of $O(\log m)$ when other parameters are regarded as constants.
\section{SL by MIL}
\subsection{Time-Series Classification with Shapelets}
In the following, we introduce a framework of time-series classification 
problem based on shapelets (i.e. SL problem). 
As mentioned in Introduction,
a time series $\btau = (\tau[1], \dots, \tau[L]) \in \Real^L$
can be identified with a bag
$B_\btau = \{(\tau[j], \ldots, \tau[j+\ell-1])
\mid 1 \leq j \leq L-\ell+1\}$ consisting of all subsequences
of $\btau$ of length $\ell$.
The learner receives a labeled sample 
$S = ((B_{\btau_1}, y_1), \ldots, (B_{\btau_m}, y_m)) \in (2^{\Real{^\ell}}
\times \{-1, 1\})^m$, 
where each labeled bag (i.e. labeled time series) 
is independently drawn 
according to some unknown distribution $D$ over a finite
support of $2^{\Real^{\ell}} \times \{-1, +1\}$.
The goal of the learner is to predict the labels of 
an unseen time series correctly.
In this way, the SL problem can be viewed as an MIL problem,
and thus we can apply our algorithms and theory.

Note that, for time-series classification, various similarity measures 
can be represented by a kernel.
For example, the Gaussian kernel (behaves like the Euclidean distance)
and Dynamic Time Warping (DTW) kernel.
Moreover, our framework can generally apply to non-real-valued sequence data 
(e.g., text, and a discrete signal) using a string kernel.

\subsection{Our Theory and Algorithms for SL}
By Theorem~\ref{theo:main}, 
we can immediately obtain the generalization bound of 
our hypothesis class in SL as follows:
\begin{coro}
Consider time-series sample $S$ of size $m$ and length $L$.
For any fixed $\ell < L$, the following generalization error bound
holds for all $g \in \conv(H_U)$ in which the length of shapelet
is $\ell$:
\[
\calE_D(g) \leq \calE_{\rho}(g) +
O\left(
 \frac{R \sqrt{d_{\Phi, S}^* \log (m(L-\ell +1))}}{\rho\sqrt{m}} \right).
\]
\end{coro}
To the best of our knowledge, this is the first result on the 
generalization performance
of SL. 
Note that the bound can also provide a theoretical justification
for some existing shapelet-based methods.
This is because many of the existing methods
find effective shapelets from all of the subsequences in the training sample,
and the linear convex combination of the hypothesis class using such shapelets
is a subset of the hypothesis class that we provided.

For time-series classification problem, shapelet-based classification
has a greater advantage of 
the interpretability or visibility
than the other
time-series classification methods \citep[see, e.g.,][]{Ye:2009:TSS:1557019.1557122}.
Although we use a nonlinear kernel function,
we can also observe important subsequences 
that contribute a shapelet by solving OP~\ref{align:WeakLearnSP} 
because of the sparsity (see also the experimental results).
Moreover, for unseen time-series data,
we can observe which subsequences contribute
the predicted class
by observing maximizer $x \in B$. 
\section{Experiments}
In the following experiments, we demonstrate that
our methods are practically effective
for time-series data and 
multiple-instance data.
Note that we use some heuristics for improving efficiency
of our algorithm in practice
(see details in supplementary materials).
We use $k$-means clustering in the heuristics,
and thus we show the average accuracies and standard deviations
for our results considering the  randomness of $k$-means.

\subsection{Results for Time-Series Data}
We used several binary labeled datasets\footnote{Our method is
applicable to multi-class classification tasks by easy expansion
\citep[e.g.,][]{NIPS1999_1773}}
in UCR datasets~\citep{UCRArchive}, which are often used as
benchmark datasets for time-series classification methods.
The detailed information of the datasets is described 
on the left-side of Table~\ref{tab:acc1}.
We used a weak learning problem OP~\ref{align:WeakLearnSP}
because the interpretability of the obtained classifier is required 
in shapelet-based time-series classification.
We set the hyper-parameters as follows:
Length $\ell$ of the subsequences 
($\ell$ corresponds to the dimension of instances in MIL)
was searched in
$\{0.1, 0.15, 0.2, 0.25, 0.3, 0.35, 0.4\}\times L$, 
where $L$ is the length of each time series in the dataset,
and we choose $\nu$ from $\{0.1, 0.2\}$.
We used the Gaussian kernel $K(x, x') = \exp(-\sigma\|x -
x'\|^2)$.
We choose $\sigma$ from $\{0.005, 0.01, 0.015, \ldots, 0.1 \}$.
We found good $\ell$, $\nu$, and $\sigma$ through a grid search via 
five runs of $5$-fold cross validation.
%
%
As an LP solver for 
WeakLearn and LPBoost we used the CPLEX software.
\paragraph{Accuracy and efficiency}
The classification accuracy results
are shown on the right-hand side of Table~\ref{tab:acc1}.
We referred to the experimental results reported by~\citet{Bagnall2017} 
with regard to accuracy of ST
method~\citep{Hills:2014:CTS:2597434.2597448}
as a baseline.
\citet{Bagnall2017} fairly compared many kinds of 
time-series classification methods and 
reported that ST achieved
higher accuracy than the other shapelet-based methods.
Our method performed better than ST for five datasets, 
but worse for the other six datasets. 
Our conjecture is that one reason for some of the 
worse results is that
ST methods consider all possible lengths ($1, \ldots, L$) 
of subsequences as shapelets 
without limiting the computational cost.
The main scheme in ST method is searching effective shapelets,
and the time complexity of it 
depends on $O(L^2m^4)$ \citep[see also the real computation time in][]{Hills:2014:CTS:2597434.2597448}.
We cannot compare the time complexity of 
our method with that of ST
because the time complexity of our method 
mainly depends on the LP solver (boosting converged in several tens of
iterations empirically).
Thus, we present the computation time 
per single learning with the best parameter
in the rightmost column of Table~\ref{tab:acc1}.
The experiments are done with a machine with 
Intel Core i7 CPU with 4 GHz and 32 GB memory.
The result demonstrated that our method efficiently ran in practice.
As a result, we can say that our method performed favorably with ST
while we limited the length of shapelets in the experiment.
\begin{table*}
\centering
\caption{Result for time-series datasets. Detailed
  information of the datasets, classification
  accuracies, and computation time (sec.) for our method. The accuracies of ST refer to the results of~\citep{Bagnall2017}.\label{tab:acc1}}
\begin{tabular}{|c |r|r|r ||r|r||r|r } \hline
dataset & \#train &  \#test & $L$  & ST & \Ourmethod &comp. time \\ \hline
Coffee & 28 & 28 & 286  & 0.995 & {\bf 1.000
                                                            $\pm$
                                        0.000} & 3.6\\ \hline
ECG200 & 100 & 100 & 96  & 0.840  &{\bf 0.877 $\pm$ 0.009} &15.9\\ \hline
ECGFiveDays &23 & 861 & 136  & 0.955  &{\bf 1.000
                                                            $\pm$ 0.000} &12.2\\ \hline
Gun-Point & 50 & 150 & 150   & {\bf 0.999}  & 
                                                                 0.976
                                                                 $\pm$ 0.006&4.7\\
  \hline
ItalyPower. & 67 & 1029 & 24  & {\bf 0.953}   &
                                                                   0.932
                                                                   $\pm$
                                                                   0.009&5.7\\
  \hline
MoteStrain & 20 & 1252 & 84   & {\bf 0.882}  &
                                                                  0.754
                                                                  $\pm$ 0.019&9.7
  \\
  \hline
ShapeletSim & 67 & 1029 & 24  & 0.934 & {\bf 0.994 $\pm$ 0.000}&11.0 
  \\
  \hline
SonyAIBO1 & 20 & 601 & 70   &  0.888 &   {\bf 0.944
                                                            $\pm$ 0.032}&3.8\\
  \hline
SonyAIBO2 & 20 & 953 & 65  &  {\bf 0.924} &  
                                                                  0.871
                                                                  $\pm$
                                             0.022&6.2 \\ \hline
ToeSeg.1 & 40 & 228 & 277  &  {\bf 0.954} &  
                                                                  0.911
                                                                  $\pm$ 0.025&20.2
  \\
  \hline
ToeSeg.2 & 36 & 130 & 343  &  {\bf 0.947} &  
                                                                  0.840
                                                                  $\pm$ 0.017&33.3
  \\ \hline
\end{tabular}
\end{table*}
\begin{table*}[t!]
\centering
\caption{Result for MIL datasets. Detailed
  information of the datasets, classification
  accuracies. \label{tab:acc2}}
\begin{tabular}{|c |r|r ||r|r|r| } \hline
dataset & sample size & \#dim. & mi-SVM w/ best kernel &
                                                                    MI-SVM
                                                                    w/
                                                                    best
  kernel & Ours w/ Gauss. kernel
  \\ \hline
MUSK1 & 92 & 166  &  0.834 $\pm$ 0.043 & 0.8335 $\pm$ 0.041 &
                                                                    {\bf
                                                                    0.8509
                                                                    $\pm$
                                                                    0.037} \\ \hline
MUSK2 & 102 & 166  & 0.736 $\pm$ 0.040& 0.840 $\pm$ 0.037& {\bf
                                                               0.8587
                                                               $\pm$ 0.038} \\ \hline
elephant & 200 &  230 & 0.802 $\pm$ 0.028& {\bf 0.822 $\pm$ 0.028} &
                                                                       0.8210
                                                                             $\pm$ 0.027
  \\ \hline
fox & 200 & 230 & 0.618 $\pm$ 0.035 & 0.581 $\pm$ 0.045 & {\bf
                                                              0.6505
                                                              $\pm$ 0.037} \\ \hline
tiger &200 & 230  & 0.765 $\pm$ 0.039& 0.815 $\pm$ 0.029& {\bf
                                                               0.8280
                                                               $\pm$ 0.024}  \\ \hline
\end{tabular}
\end{table*}
\paragraph{Interpretability of our method}
In order to show the interpretability of our method,
we introduce two types of visualization of our result.

One is the visualization of the 
characteristic subseqences of an input time series.
When we predict the label of the time series $B$, 
we calculate a maximizer $x^*$ in $B$ for each $h_\bu$,
that is, $x^*=\arg\max_{x \in B} \langle \bu, \Phi(x)\rangle$.
In image recognition tasks,
the maximizers are commonly used to observe
the sub-images that characterize the 
class of the input image~\citep[e.g.,][]{MI1normSVM}.
In time-series classification task, 
the maximizers also can be used
to observe some characteristic subsequences.
Figure~\ref{fig:gunpoint_maximizer}(a) is an example
of visualization of maximizers.
Each value in the
legend indicates $w_\bu \max_x \in B\langle\bu, \Phi(x)\rangle$.
That is, Subsequences with positive values 
contribute the positive class and subsequences 
with negative values contribute the negative class.
Such visualization provides
the subsequences that
characterize the class of the input time series.

The other is the visualization of a final hypothesis
$g(B) = \sum_{j=1}^tw_jh_j(B)$, where $h_j(B)=\max_{x \in B}\sum_{z_j \in
 \hatallins} \alpha_{j, z_j}K(z_j, x)$ ($\hatallins$ is the set
of representative subsequences, see details in supplementary
materials).
Figure~\ref{fig:gunpoint}(b)
is an example of visualization
of a final hypothesis obtained by \Ourmethod.
The colored lines are all the 
$z_j$s in $g$ where both $w_j$ and $\alpha_{j, z_j}$ were non-zero.
Each value of the legends shows the multiplication of 
$w_j$ and $\alpha_{j, z_j}$ corresponding to $z_j$.
That is, positive values on the colored lines indicate 
the contribution rate for the positive class, 
and negative values indicate the contribution rate 
for the negative class.
Note that, because it is difficult to visualize the shapelets
over the Hilbert space associated with the Gaussian kernel,
we plotted each of them to match the original time series based on 
the Euclidean distance.
Unlike visualization analyses using 
the existing shapelets-based methods
\citep[see, e.g., ][]{Ye:2009:TSS:1557019.1557122}, 
our visualization, colored lines and plotted position,
do not strictly represent the 
meaning of the final hypothesis because 
of the non-linear feature map.
However, we can say that the colored lines represent
``important patterns'', and certainly make important 
contributions to classification.
\begin{figure}[ht!]
\centering
\begin{tabular}{c}
\begin{minipage}{0.5\hsize}
\centering
  \includegraphics[width=37mm, height=33mm]{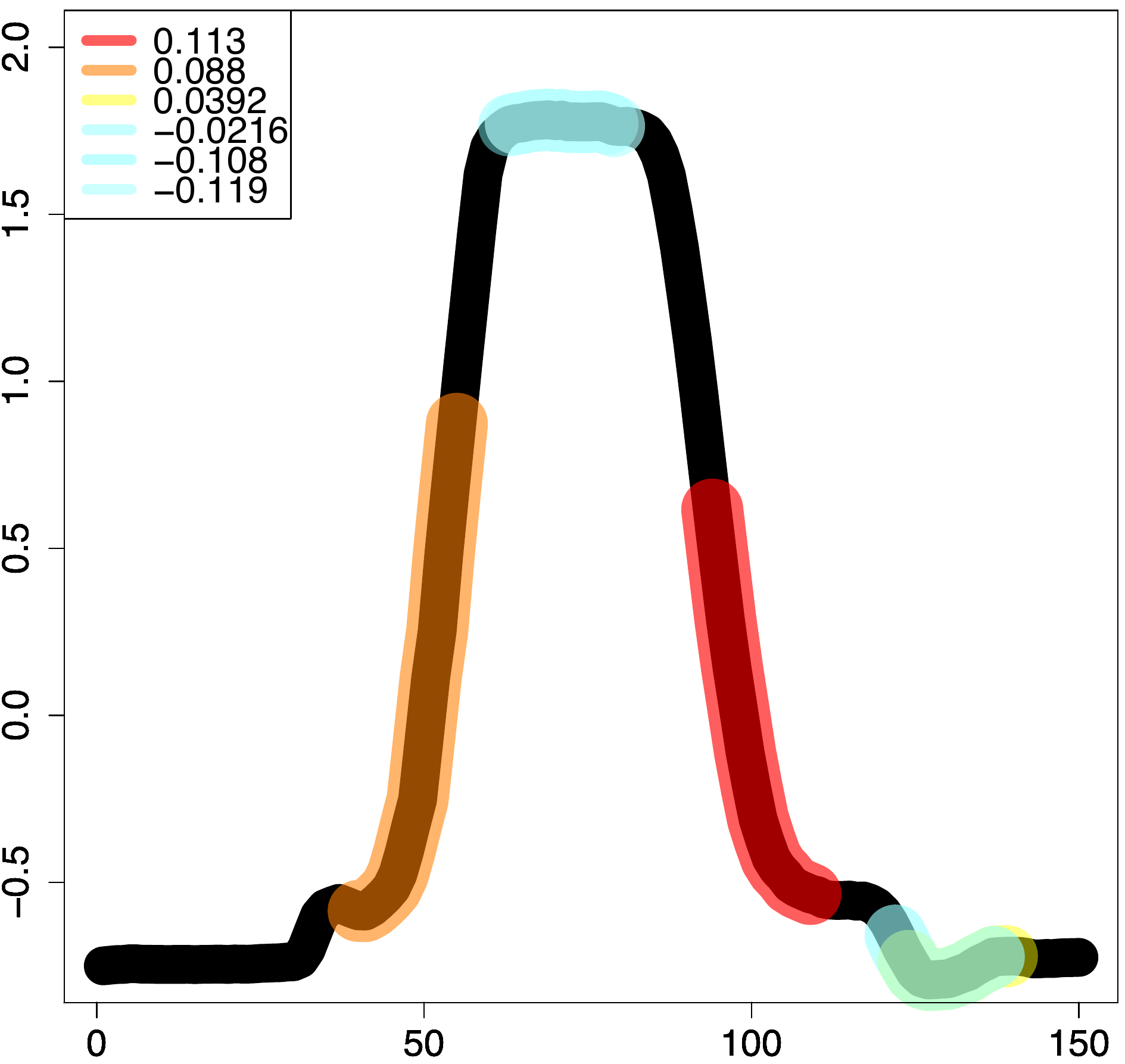}
  \\(a) 
    \end{minipage}
\begin{minipage}{0.5\hsize}
\centering
\includegraphics[width=37mm, height=33mm]{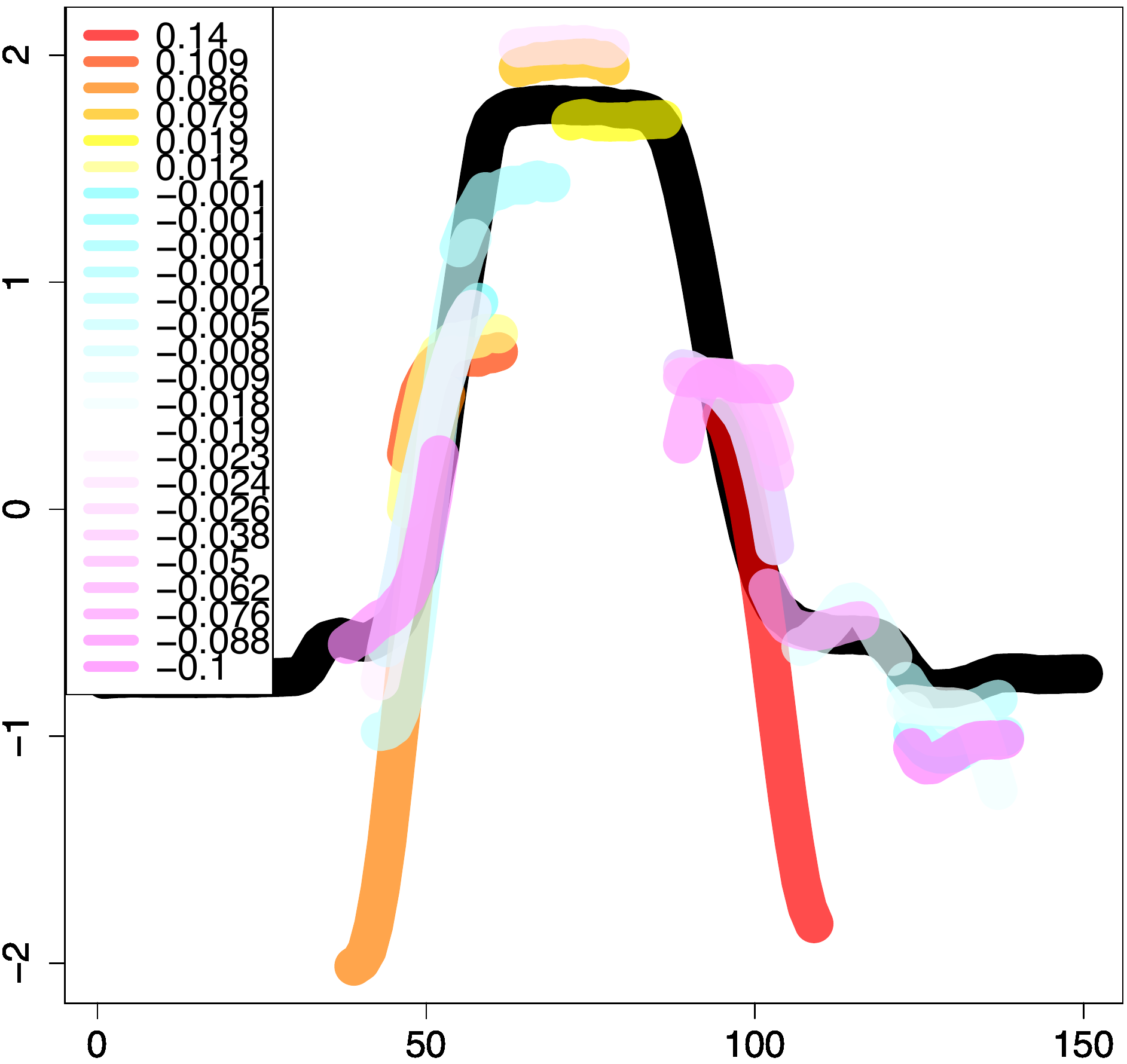}
        \hspace{1.6cm} (b) 
    \end{minipage}
\end{tabular}
\caption{Examples of the visualization for a time series of Gun-Point
   data. Both (a) and (b) images are same time series (black line)
   but the scale is different.
   (a) Each colored line is a subsequence that maximizes the
   similarity with some shapelet in a classifier. 
   Subsequences with positive values (red to yellow) contribute the positive class and
   subsequences with negative values (blue) contribute the negative
   class.
\label{fig:gunpoint_maximizer}
(b)The colored lines show important patterns of output classifier.
Positive values on the colored lines (red to yellow) 
 indicate the contribution rate for the positive class,
  and negative values (blue to purple) indicate the
  contribution rate for the negative class. \label{fig:gunpoint}}
\end{figure}
\subsection{Results for Multiple-Instance Data}
We selected the baselines of MIL algorithms as
mi-SVM and MI-SVM~\citep{NIPS2002misvm}.
Both algorithms are now classical, but still perform favorably
compared with state-of-the-art methods
for standard multiple-instance data \citep[see, e.g.,][]{doran:thesis}.
Moreover, the generalization bound of these algorithms 
are shown in~\citep{Sabato:2012:MLA}
because the algorithms obtain a (single) shapelet-based classifier.
Hence, the following comparative experiments simulate a single shapelet with
theoretical generalization ability versus infinitely many shapelets
with theoretical generalization ability.
We combined a linear, polynomial, and Gaussian kernel 
with mi-SVM and MI-SVM, respectively.
Parameter $C$ was chosen from $\{1, 10, 100, 1000, 10000\}$,
degree $p$ of the polynomial kernel is chosen from $\{2, 3, 4, 5\}$
and parameter $\sigma$ of the Gaussian kernel was chosen from $\{0.001,
0.005, 0.01, 0.05, 0.1, 0.5, 1.0\}$.
For our method, we chose $\nu$ from $\{0.5, 0.3, 0.2, 0.15, 0.1\}$,
and we only used the Gaussian kernel. 
We chose $\sigma$ from $\{0.005, 0.01, 0.05, 0.1, 0.5, 1.0 \}$.
Although we demonstrated both non-sparse and sparse weak learning,
interestingly,
sparse version beat non-sparse version for all datasets.
Thus, we will only show the result on the sparse version because of
space limitations.
For all these algorithms, we estimated optimal parameter set 
via 5-fold cross-validation.
We used well-known multiple-instance data
as shown on the left-hand side of Table~\ref{tab:acc2}.
The accuracies resulted from 10 runs of 10-fold cross-validation.

The results are shown in the right-hand side of Table~\ref{tab:acc2}.
Because of space limitations, 
for baselines
we only show the results of the kernel that achieved the best
accuracy.
Although the accuracy of our method for fox data was slightly worse, 
our method significantly outperformed 
baselines
for the other 4 datasets.


\clearpage
\appendix
\section{Supplementary materials}
\subsection{Proof of Theorem~\ref{theo:represent}}
\begin{defi}
\label{def:theta}
{\rm [The set $\Theta$ of mappings from a bag to an instance]} \\ 
Given a sample $S=(B_1, \dots, B_m)$.
For any $\bu \in U$, let $\theta_{\bu, \Phi}: \{B_1, \ldots, B_m\} \to \calX$ be a
mapping defined by
\[
\theta_{\bu, \Phi}(B_i) := \arg\max_{x \in B_i} 
 \left\langle \bu,  \Phi\left(x\right)  \right\rangle, 
\]
and we define the set of all $\theta_{\bu, \Phi}$ for $S$ as 
$\Theta_{S, \Phi} = \{\theta_{\bu, \Phi} \mid \bu \in U \}$.
For the sake of brevity, $\theta_{\bu, \Phi}$ and $\Theta_{S, \Phi}$
will be abbreviated as $\theta_{\bu}$ and $\Theta$, respectively.
\end{defi}
 \begin{proof}
We can rewrite the optimization problem (\ref{align:WeakLearn_u})
by using $\theta \in \Theta$ as follows:
\begin{align}
 \max_{\theta \in \Theta}\max_{\bu \in \Hilbert: \theta_\bu=\theta} \quad& 
\sum_{i=1}^my_id_i
\left\langle \bu,  \Phi\left(\theta(B_i)\right)  \right\rangle \\
 \nonumber
\text{sub.to} \quad& \|\bu\|_\Hilbert^2 \leq 1.
\end{align}
Thus, if we fix $\theta \in \Theta$, we have a sub-problem. Since
  the constraint $\theta=\theta_\bu$ can be written as 
 the number $|\multiallins|$ of linear constraints, 
each sub-problem is equivalent to a convex optimization.  
Indeed, each sub-problem can be written as the equivalent unconstrained
  minimization (by neglecting constants in the objective)
\begin{align*}
\min_{\bu \in \Hilbert} ~&\beta \|\bu\|^2_\Hilbert -
\sum_{i=1}^m\sum_{x \in B_i}\left(
 \eta_{i, x} \left\langle \bu, \Phi\left(\theta(B_i)\right) \right\rangle -  
\left\langle \bu, \Phi(x)  \right\rangle
\right) \\
  &-\sum_{i=1}^my_id_i
 \left\langle \bu,  \Phi\left(\theta(B_i)\right)  \right\rangle\\
\text{sub.to} ~& \langle \bu, \Phi(x) \rangle \leq \langle \bu,
                     \Phi(\theta(B_i)) \rangle \; (i \in [m], x \in B_i),
\end{align*}
 where $\beta$ and $\eta_{i,x}$ $(i \in [m], x \in B_i)$ are 
the corresponding positive constants.
Now for each sub-problem, we can apply the standard Representer Theorem
  argument (see, e.g.,~\cite{Mohri.et.al_FML}).
  Let $\Hilbert_1$ be the subspace $\{\bu \in \Hilbert \mid
 \bu=\sum_{z \in \multiallins} \alpha_{z}\Phi(z), \alpha_{z}\in \Real\}$.
 We denote $\bu_1$ as the orthogonal projection of $\bu$ onto
 $\Hilbert_1$ and any $\bu\in \Hilbert$ has the decomposition
 $\bu=\bu_1 + \bu^{\perp}$. Since $\bu^{\perp}$ is orthogonal
 w.r.t. $\Hilbert_1$, $\|\bu\|_\Hilbert^2=\|\bu_1\|_\Hilbert^2 +
 \|\bu^{\perp}\|_\Hilbert^2 \geq \|\bu_1\|_\Hilbert^2$.
On the other hand, $\left\langle \bu,  \Phi\left(z \right)
 \right\rangle =\left\langle \bu_1,  \Phi\left(z \right)  \right\rangle$.
 Therefore, the optimal solution of each sub-problem has to be contained
  in $\Hilbert_1$.
  This implies that the optimal solution, which is the maximum over all
  solutions of sub-problems, is contained in $\Hilbert_1$ as well.
 \end{proof}

\subsection{Proof of Theorem~\ref{theo:main}}
We use $\theta$ and $\Theta$ of Definition~\ref{def:theta}.

\begin{defi}
{\rm [The Rademacher and the Gaussian complexity~\cite{Bartlett:2003:RGC}]}\\
Given a sample $S=(x_1,\dots,x_m) \in \calX^m$, 
the empirical Rademacher complexity $\Rdm(H)$ of a class $H \subset
 \{h: \calX \to \Real\}$ w.r.t.~$S$
 is defined as 
 $\Rdm_S(H)=\frac{1}{m}\Expo_{\vecsigma}\left[
\sup_{h \in H}\sum_{i=1}^m \sigma_i h(x_i)
 \right]$,
 where $\vecsigma \in \{-1,1\}^m$ and each $\sigma_i$ is an independent
 uniform random variable in $\{-1,1\}$.
 The empirical Gaussian complexity $\GC_S(H)$ of $H$ w.r.t.~$S$
 is defined similarly but
 each $\sigma_i$ is drawn independently from the standard normal distribution.
\end{defi} 

The following bounds are well-known. 
\begin{lemm}
 \label{lemm:RC_and_GC}
{\rm [Lemma 4 of~\cite{Bartlett:2003:RGC}]}
 $\Rdm_S(H) =O(\GC_S(H))$.
\end{lemm}

\begin{lemm}
\label{lemm:ensemble_margin_bound}
{\rm [Corollary 6.1 of~\cite{Mohri.et.al_FML}]} 
For fixed $\rho$, $\delta >0$, 
the following bound holds with probability at least $1- \delta$:
for all $f \in \hullH$,
\[
\calE_D(f) \leq \calE_{\rho}(f) + \frac{2}{\rho} \Rdm_S(H) + 3 \sqrt{\frac{\log\frac{1}{\delta}}{2m}}.
\]
\end{lemm}

To derive generalization bound based on the Rademacher or the Gaussian
complexity is quite standard in the statistical learning theory
literature and applicable to our classes of interest as well.
However, a standard analysis provides us sub-optimal bounds. 

\begin{lemm}
\label{lemm:GC}
Suppose that for any $z \in \calX$, $\|\Phi(z)\|_\Hilbert \leq R$.
Then, the empirical Gaussian complexity of $H_U$ with respect to $S$
for $U \subseteq \{\bu \mid \|\bu\|_\Hilbert \leq 1\}$
is bounded as follows: 
\[
\GC_{S}(H) \leq \frac{R\sqrt{(\sqrt{2}-1)+ 2(\ln|\Theta|)}}{\sqrt{m}}.
\]
\end{lemm}
\begin{proof} 
Since $U$ can be partitioned into
	$\bigcup_{\theta \in \Theta} \{ \bu \in U \mid \theta_\bu = \theta\}$,
\begin{align}
\label{eq:gauss1}
\nonumber
\GC_{S}(H_U) &= \frac{1}{m} \Expo_{\bsigma} \left[\sup_{\theta \in \Theta} 
\sup_{\bu \in U:\theta_{\bu}=\theta} \sum_{i=1}^m \sigma_i
\left\langle \bu, \Phi\left(\theta(B_i)\right)
             \right\rangle\right] \\ \nonumber
&= \frac{1}{m} \Expo_{\bsigma} \left[\sup_{\theta \in \Theta} 
\sup_{\bu \in U:\theta_{\bu}=\theta} \left\langle \bu,
 \left(\sum_{i=1}^m \sigma_i \Phi\left(\theta(B_i)\right)
  \right) \right \rangle \right] \\ \nonumber
&\leq \frac{1}{m} \Expo_{\bsigma} \left[\sup_{\theta \in \Theta} 
\sup_{\bu \in U} \left\langle \bu, 
\left(\sum_{i=1}^m \sigma_i \Phi\left(\theta(B_i)\right)
  \right) \right \rangle \right] \\ \nonumber
  &\leq \frac{1}{m} \Expo_{\bsigma} \left[
  \sup_{\theta \in \Theta} \left\| \sum_{i=1}^m \sigma_i \Phi \left(
  \theta(B_i)\right) \right\|_\Hilbert
  \right]\\ \nonumber
  &= \frac{1}{m} \Expo_{\bsigma} \left[
  \sup_{\theta \in \Theta} \sqrt{\left\|\sum_{i=1}^m \sigma_i
  \Phi\left(\theta(B_i)\right)\right\|_\Hilbert^2}
  \right] \\ \nonumber
  &= \frac{1}{m} \Expo_{\bsigma} \left[
  \sqrt{\sup_{\theta \in \Theta} \left\|\sum_{i=1}^m \sigma_i
  \Phi\left(\theta(B_i)\right)\right\|_\Hilbert^2}
  \right]\\ 
& \leq \frac{1}{m} \sqrt{ 
  \Expo_{\bsigma} \left[
  \sup_{\theta \in \Theta} \left\|\sum_{i=1}^m \sigma_i
  \Phi\left(\theta(B_i)\right)\right\|_\Hilbert^2
  \right]}.
\end{align}
The first inequality is derived from the relaxation of $\bu$,
the second inequality is due to Cauchy-Schwarz inequality and
the fact $\|\bu\|_{\Hilbert} \leq 1$, and
the last inequality is due to Jensen's inequality.
We denote by $\kernel^{(\theta)}$ the kernel matrix such that 
$\kernel_{ij}^{(\theta)} = \langle \Phi((\theta(B_i)),
\Phi(\theta(B_j)) \rangle$.
 Then, we have
 \begin{align}
  \Expo_{\bsigma} \left[
  \sup_{\theta \in \Theta} \left\|\sum_{i=1}^m \sigma_i
  \Phi\left(\theta(B_i)\right)\right\|_\Hilbert^2
  \right]
  =
  \Expo_{\bsigma}\left[ \sup_{\theta \in \Theta}  
\sum_{i,j=1}^m \sigma_i\sigma_j 
\kernel_{ij}^{(\theta)}
\right].
 \end{align}
We now derive an upper bound of the r.h.s. as follows.
 
For any $c>0$, 
\begin{align*}
&\exp\left(
c \Expo_{\bsigma}\left[ \sup_{\theta \in \Theta}  
\sum_{i,j=1}^m \sigma_i\sigma_j 
\kernel_{ij}^{(\theta)}
\right]
\right) \\
&\leq 
\Expo_{\bsigma}\left[\exp\left(
c \sup_{\theta \in \Theta}  
\sum_{i,j=1}^m \sigma_i\sigma_j 
\kernel_{ij}^{(\theta)}
\right)
\right] \\
= &
\Expo_{\bsigma}\left[
\sup_{\theta \in \Theta}  
\exp\left(c
\sum_{i,j=1}^m \sigma_i\sigma_j 
\kernel_{ij}^{(\theta)}
\right)
\right] \\
&\leq 
\sum_{\theta \in \Theta}
\Expo_{\bsigma}\left[
\exp\left(c
\sum_{i,j=1}^m \sigma_i\sigma_j 
\kernel_{ij}^{(\theta)}
\right)
\right]
\end{align*}
The first inequality is due to Jensen's inequality, and
the second inequality is due to the fact that the supremum is
bounded by the sum.
By using the symmetry property of $\kernel^{(\theta)}$, we have
$\sum_{i,j=1}^m\sigma_i\sigma_j\kernel_{ij}^{(\theta)}=\T{\bsigma}\kernel^{(\theta)}\bsigma$,
which is rewritten as
\[
\T{\bsigma}\kernel^{(\theta)}\bsigma = \T{(\T{\Vmat}\bsigma)}
\left( \begin{array}{ccc}
\lambda_1^{(\theta)} & \hfill & 0  \\
\hfill & \ddots & \hfill  \\
0 & \hfill & \lambda_m^{(\theta)} \\
\end{array} \right)
\T{\Vmat}\bsigma,
\]
where
$\lambda_1^{(\theta)}\geq \dots \geq \lambda_m^{(\theta)}\geq 0$ are the
eigenvalues of $\kernel^{(\theta)}$ and 
$\Vmat = (\bv_1, \ldots, \bv_m)$
is the orthonormal matrix such that $\bv_i$ is the eigenvector
that corresponds to the eigenvalue $\lambda_i$.
By the reproductive property of Gaussian distribution,
$\T{\Vmat} \bsigma$ obeys the same Gaussian distribution as well.
So,
\begin{align*}
&\sum_{\theta \in \Theta}
\Expo_{\bsigma}\left[
\exp\left(c
\sum_{i,j=1}^m \sigma_i\sigma_j 
\kernel_{ij}^{(\theta)}
\right)
\right] \\
=&
\sum_{\theta \in \Theta}
\Expo_{\bsigma}\left[
\exp\left(c
\T{\bsigma}\kernel^{(\theta)}\bsigma
\right)
\right] \\
= &
\sum_{\theta \in \Theta}
\Expo_{\bsigma}\left[
\exp\left(c
\sum_{k=1}^m\lambda_k^{(\theta)}(\T{\bv_k}\bsigma)^2
\right)
\right] \\
= &
\sum_{\theta \in \Theta}
\Pi_{k=1}^m
\Expo_{\sigma_k}\left[
\exp\left(c
\lambda_k^{(\theta)}\sigma_k^2
\right)
\right] ~~(\text{replace}~\bsigma = \T{\bv_k}\bsigma)\\
 = &
 \sum_{\theta \in \Theta}
\Pi_{k=1}^m
 \left(
 \int_{-\infty}^{\infty}
 \exp\left(
 c \lambda_k^{(\theta)}\sigma^2
 \right)
  \frac{\exp(-\sigma^2)}{\sqrt{2\pi}}d\sigma
\right) \\
 = &
\sum_{\theta \in \Theta}
\Pi_{k=1}^m
 \left(
  \int_{-\infty}^{\infty}
  \frac{\exp(-(1-c\lambda_k^{(\theta)})\sigma^2)}{\sqrt{2\pi}}d\sigma\right).
\end{align*}
Now we replace $\sigma$ by $\sigma' =
\sqrt{1-c\lambda_k^{(\theta)}}\sigma$. Since
$d\sigma'=\sqrt{1-c\lambda_k^{(\theta)}}d\sigma$,
we have:
\begin{align*}
&\int_{-\infty}^{\infty}
\frac{\exp(-(1-c\lambda_k^{(\theta)})\sigma^2)}{\sqrt{2}\pi}d\sigma 
= \frac{1}{\sqrt{2\pi}} \int_{-\infty}^{\infty} 
\frac{\exp(-\sigma'^2)}{\sqrt{1-c\lambda_k^{(\theta)}}}d\sigma' \\
= &\frac{1}{\sqrt{1-c\lambda_k^{(\theta)}}}.
\end{align*}
 Now,  applying the inequality that $\frac{1}{\sqrt{1-x}} \leq
 1+2(\sqrt{2}-1)x$ for  $0 \leq x \leq \frac{1}{2}$, 
 the bound becomes
 \begin{align}
\label{align:ineq1}
\nonumber
 &\exp\left(
c \Expo_{\bsigma}\left[ \sup_{\theta \in \Theta}  
\sum_{i,j=1}^m \sigma_i\sigma_j 
\kernel_{ij}^{(\theta)}
\right]
  \right) \\
\leq
&\sum_{\theta \in \Theta}
\Pi_{k=1}^m
  \left(
1+2(\sqrt{2}-1)c\lambda_k^{(\theta)} + 2\lambda_1
  \right).
 \end{align}
Further, taking logarithm, dividing the both sides by $c$,
letting $c=\frac{1}{2
  \max_{k}\lambda_k^{(\theta)}}=1/(2\lambda_1^{(\theta)})$,
fix $\theta = \theta^*$ such that $\theta^*$ maximizes (\ref{align:ineq1}),
 and applying
 $\ln(1+x) \leq x$, we get:
\begin{align}
\label{eq:gauss2}
\nonumber
&\Expo_{\bsigma}\left[ \sup_{\theta \in \Theta}  
\sum_{i,j=1}^m \sigma_i\sigma_j 
\kernel_{ij}^{(\theta^*)}
\right]
\\ \nonumber
&\leq
 (\sqrt{2}-1)\sum_{k=1}^m \lambda_k^{(\theta^*)}  + 2\lambda_1^{(\theta^*)} \ln |\Theta|\\
\nonumber
 &=
 (\sqrt{2}-1)\tr(\kernel^{(\theta^*)}) + 2\lambda_1^{(\theta^*)} \ln |\Theta|\\
 &\leq
 (\sqrt{2}-1)m R^2 + 2 mR^2 \ln |\Theta|,
\end{align}
 where the last inequality holds since $\lambda_1^{(\theta^*)}=\|\kernel^{(\theta^*)}\|_2 \leq m
 \|\kernel^{(\theta)}\|_{\max} \leq R^2$.
By equation~(\ref{eq:gauss1}) and (\ref{eq:gauss2}), 
we have:
\begin{align*}
\GC_S(H) 
&\leq
\frac{1}{m}
\sqrt{
\Expo_{\bsigma}
\left[\sup_{\theta \in \Theta}  
\sum_{i,j=1}^m \sigma_i\sigma_j 
\kernel_{ij}^{(\theta)}
\right]} \\
&\leq 
\frac{R \sqrt{(\sqrt{2}-1) + 2\ln 
|\Theta|
}}{\sqrt{m}}.
\end{align*}
\end{proof}

Thus, it suffices to bound the size $|\Theta|$. 
The basic idea to get our bound is the following geometric analysis.
Fix any $i \in [m]$ and consider points $\{\Phi(x) \mid x \in B_i\}$.
Then, we define equivalence classes of $\bu$ such that
$\theta_\bu(i)$ is in the same class, which define a Voronoi diagram
for the points $\{\Phi(x) \mid x \in B_i\}$.
Note here that the similarity is measured by the inner product, not a
distance. More precisely, 
let $V(B_i) = \{V(x) \mid x \in B \}$ be the Voronoi diagram,
each of the region is defined as 
$V(x) = \{\bu \in \Hilbert \mid \theta_{\bu}(B_i)=x\}$
Let us consider the set of intersections
$\bigcap_{i \in [m]}V_i{(x_i)}$ for all combinations of
$(x_1, \ldots, x_m) \in B_1 \times \cdots \times B_m $.
The key observation is that each non-empty intersection corresponds to a
mapping $\theta_{\bu} \in \Theta$. Thus,
we obtain $|\Theta|=(\text{the number of intersections $\bigcap_{i\in
[m]}V_{i}(x_i)$})$. In other words, the size of $\Theta$ is exactly
the number of rooms defined by the intersections of $m$ Voronoi
 diagrams $V_1,\dots,V_m$.
 From now on, we will derive upper bound based on this observation.

 \begin{lemm}
  \label{lemm:Theta}
 \[
  |\Theta| =O(|\multiallins|^{2d_{\Phi,S}^*}).
 \]
 \end{lemm}
\begin{proof}
 We will reduce the problem of counting intersections of the Voronoi
 diagrams to that of counting possible labelings by hyperplanes for
 some set.
 Note that for each neighboring Voronoi regions, the border is a part of
 hyperplane since the closeness is defined in terms of the inner
 product.
 Therefore, by simply extending each border to a hyperplane, we obtain
 intersections of halfspaces defined by the extended hyperplanes.
 Note that, the size of these intersections gives an upper bound of
 intersections of the Voronoi diagrams.
 More precisely, we draw hyperplanes for each pair of points in
 $\Phi(\multiallins)$
 so that each point on
 the hyperplane has the same inner product between two points.
 Note that for each pair $\Phi(z), \Phi(z') \in \multiallins$, the normal vector of
 the hyperplane is given as $\Phi(z) -\Phi(z')$ (by fixing the sign arbitrary).
 Thus, the set of hyperplanes obtained by this procedure is exactly
 $\pdiff$. The size of $\pdiff$ is ${|\multiallins|} \choose 2$, which is at most $|\multiallins|^2$.
 Now, we consider a ``dual'' space by viewing each hyperplane as a point and each point in $U$ as a
 hyperplane.
 Note that points $\bu$ (hyperplanes in the dual) in an intersection 
 give the same labeling on the points in the dual domain.
 Therefore, the number of intersections in the original domain is the
 same as the number of the possible labelings on $\pdiff$ by hyperplanes
 in $U$. By the classical Sauer's
 Lemma and the VC dimension of hyperplanes~(see, e.g., Theorem 5.5
 in~\cite{LK02}), the size is at most $O((|\multiallins|^2)^{d_{\Phi,S}^*})$.
\end{proof}


  \begin{theo}\noindent \mbox\\
   \label{theo:Theta} 
  \begin{enumerate}
   \item[(i)] For any $\Phi$, $|\Theta|=O(|\multiallins|^{8(R/\mu^*)^2})$.
   \item[(ii)] if $\calX \subseteq \Real^\ell$ and $\Phi$ is the identity mapping over $\multiallins$, then
	      $|\Theta|=O( |\multiallins|^{\min\{8(R/\mu^*)^2, 2\ell \}}\})$.
   \item[(iii)] if $\calX \subseteq \Real^\ell$ and $\Phi$ satisfies that $\left\langle\Phi(z), \Phi(x) \right\rangle$ is monotone
decreasing with respect to $\|z-x\|_2$  (e.g.,
	the mapping defined by the Gaussian kernel) and
 $U=\{\Phi(z) \mid z \in \calX \subseteq \Real^\ell, \|\Phi(z)\|_\Hilbert \leq 1\}$, 
then $|\Theta|=O( |\multiallins|^{\min\{8(R/\mu^*)^2, 2\ell \}}\})$.
  \end{enumerate}
  \end{theo}

\begin{proof}
 (i)
 We follow the argument in Lemma 4.
 For the set of classifiers $F=\{f:\pdiff \to \{-1,1\} \mid
 f=\sign(\langle \bu, \bv\rangle), \|\bu\|_\Hilbert \leq 1, 
 \min_{\bv \in \pdiff}| \langle \bu, \bv \rangle|=\mu \}$,
 its VC dimension is known to be at most $R^2/\mu^2$
 for $\pdiff \subseteq \{\bv \mid \|\bv\|_\Hilbert \leq 2R\}$ (see, e.g.,~\cite{LK02}).
 By the definition of $\mu^*$,
 for each intersections given by hyperplanes, there always exists a point $\bu$
 whose inner product between each hyperplane is at least $\mu^*$.
 Therefore, the size of the intersections is bounded by the number of possible
 labelings in the dual space by $U''=\{\bu \in \Hilbert \mid
 \|\bu\|_\Hilbert \leq 1,  \min_{\bv \in \pdiff}| \langle \bu,
 \bv \rangle|= \mu^* \}$.
 Thus we obtain that $d_{\Phi, S}^*$ is at most $8(R / \mu^*)^2$ and by
 Lemma 4, we complete the proof of case (i).

 (ii) In this case, the Hilbert space $\Hilbert$ is contained in
 $\Real^\ell$. Then, by the fact that VC dimension $d_{\Phi,S}^*$ is at
 most $\ell$ and Lemma 4, the statement holds.

 (iii)
If $\left\langle\Phi(z), \Phi(x) \right\rangle$ is monotone
decreasing for $\|z-x\|$,
then the following holds:
\[
\arg\max_{x \in \calX} \left\langle\Phi(z), \Phi(x) \right\rangle = \arg\min_{x \in \calX} \|z-x\|_2.
\]
Therefore, $\max_{\bu:\|\bu\|_\Hilbert = 1}\langle \bu, \Phi(x)
\rangle = \|\Phi(x)\|_\Hilbert$, where $\bu =
\frac{\Phi(x)}{\|\Phi(x)\|_\Hilbert}$. 
It indicates that the number of Voronoi cells
made by $V(x)=\{z \in \Real^\ell \mid z=\arg\max_{x \in B}
  (z \cdot x) \}$ corresponds to the
$\hat{V}(x)=\{\Phi(z) \in \Hilbert \mid z=\arg\max_{x \in B}
 \langle \Phi(z), \Phi(x) \rangle \}$.
 Then, by following the same argument for the linear kernel case, we get
 the same statement.
\end{proof}
Now we are ready to prove Theorem~\ref{theo:main}.
\begin{proof}[Proof of Theorem~\ref{theo:main}]
By using Lemma~\ref{lemm:RC_and_GC}, and
 \ref{lemm:ensemble_margin_bound},
we obtain the generalization bound in terms of the Gaussian
 complexity of $H$. Then, by applying Lemma~\ref{lemm:GC} and Theorem~\ref{theo:Theta},
we complete the proof.
\end{proof}

\subsection{Heuristics for Improving Efficiency}
For large multiple-instance data such as time-series data that may 
contain a lot of subsequences,
we need to reduce the computational cost of weak learning
in Algorithm~\ref{alg:WeakLearn}.
Therefore, we introduce two heuristic options for improving efficiency.
First, in Algorithm~\ref{alg:WeakLearn}, we fix an initial
$\balpha_{0}$ as following:
More precisely, we initially solve 
\begin{align*}
\balpha_0 = \arg\max_{\balpha}& \sum_{i=1}^md_iy_i\max_{x \in B_i} \sum_{z \in \allins}
   \alpha_{z} K\left(z, x \right), \\
   ~~\text{sub.to}&~~ \balpha ~\text{is a one-hot vector.}
\end{align*}
That is, we choose the most discriminative shapelet from $\allins$
as the initial point of $\bu$ for given $\bd$.
We expect that it will speed up the convergence 
of the loop of line 3, and the obtained classifier is better 
than the methods that choose effective 
shapelets from subsequences.
For Gun-Point data described in Table~\ref{tab:acc1},
this method obtains the solution approximately seven times faster
than when random vectors are used as the initial $\balpha_0$.

Second, we reduce the high computational cost induced by 
calculating $\sum_{z \in \allins} d_k\alpha_{z} K\left(z, x\right)$
in our problem.
For example, when we 
consider subsequences as instances for time series classification, 
we have a large computational cost because of the number of subsequences 
of training data (e.g., approximately $10^6$ when sample size is $1000$ and
length of each time series is $1000$, 
which results in a similarity matrix of size $10^{12}$). 
However, in most cases, many subsequences in time series data 
are similar to each other.
Therefore, we only use representative instances $\hatallins$
instead of the set of all instances $\allins$. 
In these experiments, we extract $\hatallins$ via
$k$-means clustering of $\allins$.
Although this approach may decrease the classification accuracy,
it drastically decreases the computational cost for a large dataset.
For the Gun-Point dataset, 
this approach with $k=10$ still achieves high classification accuracy and
it is over $2000$ times faster on average than when all of subsequences
are used.
In our experiments including the MIL task, 
we use $100$-means clustering to obtain the representative instances.
Surprisingly, as demonstrated in the experiments, these 
heuristics work well not only for time-series data but also 
for multiple-instance data in practice.
\end{document}